\documentclass{article}
\usepackage{ijcai22}

\usepackage{times}
\usepackage{soul}
\usepackage{url}
\usepackage[hidelinks]{hyperref}
\usepackage[utf8]{inputenc}
\usepackage[small]{caption}
\usepackage{graphicx}
\usepackage{amsmath}
\usepackage{amsthm}
\usepackage{booktabs}
\usepackage{algorithm,algpseudocode}
\usepackage{graphicx}
\usepackage{mathtools}
\usepackage{subcaption}
\usepackage{amssymb}
\usepackage{tikz}
\usepackage{pgfplots}
\usepackage[nomessages]{fp}
\usepackage{pgfplots}
\usepackage{pgfkeys}
    \def\addlegendimage{\csname pgfplots@addlegendimage\endcsname}
\definecolor{myteal}{RGB}{0,153,136}
\definecolor{myred}{RGB}{204,51,17}
\definecolor{mygreen}{RGB}{204,221,170}
\pgfplotsset{compat=1.10, 
cycle list={%
{draw=myteal,mark=o},
{draw=myred, mark=o}}}

\newtheorem{theorem}{Theorem}
\newtheorem{definition}{Definition}
\newtheorem{Claim}{Claim}
\algnewcommand{\LineComment}[1]{\State \textcolor{blue}{/* #1 */}}
\newcommand{\ouralgo}{\texttt{P-FCB}}
\title{Differentially Private Federated Combinatorial Bandits with Constraints\thanks{A version of this paper has appeared in the Proceedings of the 32nd European Conference on Machine Learning and Principles and Practice of Knowledge Discovery in Databases, 2022 (ECML PKDD '22).}}

\author{
Sambhav Solanki
\and
Samhita Kanaparthy\and
Sankarshan Damle\And
Sujit Gujar
\affiliations
Machine Learning Lab, International Institute of Information Technology (IIIT), Hyderabad, India
\emails
\{sambhav.solanki, s.v.samhita, sankarshan.damle\}@research.iiit.ac.in,
sujit.gujar@iiit.ac.in,
}

\begin{document}

\maketitle

\begin{abstract}
There is a rapid increase in the cooperative learning paradigm in online learning settings, i.e., \emph{federated learning} (FL). Unlike most FL settings, there are many situations where the agents are competitive. Each agent would like to learn from others, but the part of the information it shares for others to learn from could be sensitive; thus, it desires its \emph{privacy}. This work investigates a group of agents working concurrently to solve similar combinatorial bandit problems while maintaining quality constraints. Can these agents collectively learn while keeping their sensitive information confidential by employing differential privacy? We observe that communicating can reduce the \emph{regret}. However, differential privacy techniques for protecting sensitive information makes the data noisy and may deteriorate than help to improve regret. Hence, we note that it is essential to decide \emph{when to communicate} and \emph{what shared data to learn} to strike a functional balance between regret and privacy. For such a federated combinatorial MAB setting, we propose a Privacy-preserving Federated Combinatorial Bandit algorithm, \ouralgo. We illustrate the efficacy of \ouralgo\ through simulations. We further show that our algorithm provides an improvement in terms of regret while upholding quality threshold and meaningful privacy guarantees.

\end{abstract}

\section{Introduction}\label{sec:intro}

A large portion of the manufacturing industry follows the Original Equipment Manufacturer (OEM) model. In this model, companies (or aggregators) that design the product usually procure components required from an available set of OEMs. Foundries like TSMC, UMC, and GlobalFoundries handle the production of components used in a wide range of smart electronic offerings~\cite{foundries}. We also observe a similar trend in the automotive industry~\cite{OEM}.

However, aggregators are required to maintain minimum \emph{quality} assurance for their products while maximizing their revenue. Hence, they must judicially procure the components with desirable quality and cost from the OEMs. For this, aggregators should learn the quality of components provided by an OEM. OEM businesses often have numerous agents engaged in procuring the same or similar components. In such a setting, one can employ \emph{online learning} where multiple aggregators, referred henceforth as \emph{agents}, cooperate to learn the qualities~\cite{FedUCB,Fed2-UCB}. Further, decentralized (or federated) learning is gaining traction for large-scale applications~\cite{google-fedml2,appl-of-fedml2}. 

In general, an agent needs to procure and utilize the components from different OEMs (referred to as \emph{producers}) to learn their quality. This learning is similar to the exploration and exploitation problem, popularly known as \emph{Multi-armed Bandit} (MAB)~\cite{adaptive,shweta}. It needs sequential interactions between sets of producers and the learning agent. Further, we associate qualities, costs, and capacities with the producers for each agent. We model this as a combinatorial multi-armed bandit (CMAB)~\cite{CUMB} problem with assured qualities~\cite{shweta}. Our model allows the agents to maximize their revenues by communicating their history of procurements to have better estimations of the qualities. Since the agents can benefit from sharing their past quality realizations, we consider them engaged in a \textit{federated} learning process. Federated MAB often improves performance in terms of \emph{regret} incurred per agent~\cite{FedCM,PF-UCB}\footnote{Regret is the deviation of utility gained while engaging in learning from the utility gained if the mean qualities were known.}.

Such a federated exploration/exploitation paradigm is not just limited to selecting OEMs. It is useful in many other domains such as stocking warehouse/distribution centres, flow optimization, and product recommendations on e-commerce websites~\cite{supply-chain,silva2022multi}. However, agents are competitive; thus, engaging in federated learning is not straightforward. Agents may not be willing to share their private experiences since that could negatively benefit them. For example, sharing the exact procurement quantities of components specific to certain products can reveal the market/sales projections. Thus, we desire (or many times it is necessary) to maintain privacy when engaged in federated learning. This paper aims to design a privacy-preserving algorithm for federated CMAB with quality assurances.

\smallskip
\noindent\textbf{Our Approach and Contributions.}
Privacy concerns for sensitive information pose a significant barrier to adopting federated learning. To preserve the privacy of such information, we employ the strong notion of \emph{differential privacy} (DP)~\cite{dwork06}. DP has emerged as the standard approach for privacy preservation in the AI/ML literature~\cite{abadi2016deep,papernot2016semi,triastcyn2019federated,ai2,pgibbs,padala2021federated}.
Note that naive approaches (e.g., Laplace or Gaussian Noise Mechanisms~\cite{Dwork}) to achieve DP for CMAB may come at a high privacy cost or outright perform worse than non-federated solutions. Consequently, the primary challenge is carefully designing methods to achieve DP that provide meaningful privacy guarantees while performing significantly better than its non-federated counterpart.

To this end, we introduce \ouralgo, a \underline{P}rivacy-preserving \underline{F}ederated \underline{C}ombinatorial \underline{B}andit algorithm. \ouralgo\ comprises a novel communication algorithm among agents, while each agent is learning the qualities of the producers to cooperate in the learning process. Crucially in \ouralgo, the agent only communicates within a specific time frame -- since it is not beneficial to communicate in (i) earlier rounds (estimates have high error probability) or (ii) later rounds (value added by communicating is minimal). While communicating in each round reduces per agent regret, it results in a high privacy loss. \ouralgo\ strikes an effective balance between learning and privacy loss by limiting the number of rounds in which agents communicate. Moreover, to ensure the privacy of the shared information, the agents add calibrated noise to sanitize the information a priori. \ouralgo\ also uses error bounds generated for UCB exploration~\cite{Auer-2002} to determine if shared information is worth learning. We show that \ouralgo\ allows the agents to minimize their regrets while ensuring strong privacy guarantees through extensive simulations.

In recent times, research has focused on the intersection of MAB and DP~\cite{P2B,DP-MAB}. Unlike \ouralgo, these works have limitations to single-arm selections. To the best of our knowledge, this paper is the first to simultaneously study federated CMAB with assured quality and privacy constraints. In addition, as opposed to other DP and MAB approaches~\cite{FedUCB,Privacy-Preserving}, we consider the sensitivity of attributes specific to a producer-agent set rather than the sensitivity of general observations. In summary, our contributions in this work are as follows:
\begin{enumerate}
    \item We provide a theoretical analysis of improvement in terms of regret in a non-private homogeneous federated CMAB setting (Theorem~\ref{theorem:Homogenous Regret}, Section~\ref{sec:np}).
    \item We show that employing privacy techniques naively is not helpful and has information leak concerns (Claim~\ref{claim}, Section~\ref{sub:infoLeak}).
    \item We introduce \ouralgo\ to employ privacy techniques practically (Algorithm~\ref{algo:Multiplayer Private SS-UCB}). \ouralgo\ includes selecting the information that needs to be perturbed and defining communication rounds to provide strong privacy guarantees. The communicated information is learned selectively by using error bounds around current estimates. Selective communication helps minimize regret.
    \item \ouralgo's improvement in per agent regret even in a private setting compared to individual learning is empirically validated through extensive simulations (Section~\ref{sec:exp}).
\end{enumerate}

\section{Related Work} \label{sec:related}

\textit{Multi-armed bandits} (MAB) and their variants are a well studied class of problems~\cite{Auer-2002,history,shweta,LinUCB,Knowledge,Routine} that tackle the exploration vs. exploitation trade-off in online learning settings. While the classical MAB problem~\cite{Auer-2002,Intro-MAB} assumes single arm pull with stochastic reward generation, our work deals with combinatorial bandits (CMAB)~\cite{CUMB,Radio,Shweta_Sujit,TS-CMAB}, whereby the learning agent pulls a subset of arms. We remark that our single-agent (non-federated) MAB formulation is closely related to the MAB setting considered in~\cite{ayush}, but the authors there do not consider federated learning.

\smallskip
\noindent\textit{Federated MAB.}
Many existing studies address the MAB problem in a federated setting but restrict themselves to single-arm pulls. ~\cite{Fed2-UCB,PF-UCB} considers a federated extension of the stochastic single player MAB problem, while ~\cite{Fed-PE} considers the linear contextual bandit in a federated setting. Kim et al.~\cite{FedCM} specifically considers the federated CMAB setting, but does not address any privacy concerns.

\smallskip
\noindent\textit{Privacy-preserving MAB.}
The authors in~\cite{P2B,DP-MAB} consider a differentially private MAB setting for a single learning agent, while the works in \cite{IOT,CDP-MAB} consider differentially private federated MAB setting. However, these works focus only on the classical MAB setting, emphasising the communication bottlenecks. There also exists works that deal with private and federated setting for the contextual bandit problem~\cite{FedUCB,Privacy-Preserving}. However, they do not consider pulling subsets of arms. Further, Hannun et al.~\cite{Privacy-Preserving} consider privacy over the context, while Dubey and Pentland~\cite{FedUCB} consider privacy over context and rewards. Contrarily, this paper considers privacy over the procurement strategy used.

To the best of our knowledge, we are the first to propose a solution for combinatorial bandits (CMAB) in a federated setting with the associated privacy concerns.

\section{Preliminaries} \label{sec:prm}
In this section, we formally describe the combinatorial multi-armed bandit setting and its federated extension. We also define differential privacy in our context.

\subsection{Federated Combinatorial Multi Armed Bandits} \label{sub:MAB}
We consider a combinatorial MAB (CMAB) setting where there are $[m]$ producers and $[n]$ agents. Each producer $i \in [m]$ has a cost $c_{ij}$ and capacity $k_{ij}$ for every agent $j \in [n]$ interacting with it. At any round $t \in \{1,2,\ldots,T\}$, agents procure some quantity of goods from a subset of producers under given constraint(s). We denote the procurement vector of an agent $j$ by $\mathbf{s}_{j} = (l_{1j}, l_{2j},\ldots,l_{mj})$ where $\ l_{ij} \in [0,k_{ij}]$ is the quantity procured from producer $i$.

\smallskip
\noindent\textit{Qualities.} Each agent observes a quality realisation for each unit it procured from producers. Since the quality of a single unit of good may not be easily identifiable, we characterize it as a Bernoulli random variable. This simulates if a unit was defective or not in the OEMs scenario. The expected realisation of a unit procured from a producer $i$ is referred to as its quality, $q_{i}$. In other words, $q_{i}$ denotes the probability with which a procured unit of good from producer $i$ will have a quality realisation of one. While the producer's cost and capacity vary across agents, the quality values are indifferent based on agents.

\smallskip
\noindent\textit{Regret.}
We use $r_{ij}$ to denote expected utility gain or revenue for the agent $j$ by procuring a single unit from producer $i$, where $r_{ij} = \rho q_{i} - c_{ij}$ (where $\rho>0$, is a proportionality constant). Further, the expected revenue for a procurement vector $\mathbf{s}_{j}$, is given by $r_{\mathbf{s}_{j}} = \sum_{i\in [m]} l_{ij}r_{ij}$.

The goal for the agent is to maximise its revenue, under given constraints. We consider a constraint of maintaining a minimum expected quality threshold $\alpha$ (quality constraint), for our setting. To measure the performance of an a given algorithm $A$, we use the notion of regret which signifies the deviation of the algorithm from the procurement set chosen by an Oracle when mean qualities are known. For any round $t \in \{1,2,\ldots,T\}$, we use the following to denote the regret for agent $j$ given an algorithm $A$,
\begin{align*}
    \mathcal{R}_{Aj}^{t} = \begin{cases}
        r_{\mathbf{s^{*}_{j}}} - r_{\mathbf{s}_{Aj}^{t}}, & \textit{if }s_{Aj}^{t} \textit{ satisfies the quality constraint} \\
        L, & \textit{otherwise}
    \end{cases}
\end{align*}
where $\mathbf{s^{*}_{j}}$ denotes the procurement set chosen by an Oracle, with the mean qualities known. $\mathbf{s}_{A}^{t}$ is the set chosen by the algorithm $A$ in round $t$. $L = \max_{r_{\mathbf{s}}}(r_{\mathbf{s^{*}_{j}}} - r_{\mathbf{s}})$ is a constant that represents the maximum regret one can acquire. The overall regret for algorithm $A$ is given by $\mathcal{R}_{A} = \sum_{j \in [n]} \sum_{t \in [T]} \mathcal{R}_{Aj}^{t}$. 

\smallskip
\noindent\textit{Federated Regret Ratio (FRR).} We introduce FRR to help quantify the reduction in regret brought on by engaging in federated learning. FRR is the ratio of the regret incurred by an agent via a federated learning algorithm $A$ over agent's learning individually via a non-federated algorithm $NF$, i.e., $FRR = \frac{\mathcal{R}_{A}}{\mathcal{R}_{NF}}$.


Observe that, $FRR \approx 1$ indicates that there is not much change in terms of regret by engaging in federated learning. If $FRR > 1$, it is detrimental to engage in federated learning, whereas if $FRR < 1$, it indicates a reduction in regret. When $FRR\approx 0$, there is almost complete reduction of regret in federated learning.

\smallskip
In our setting, we consider that agents communicate with each other to improve their regret. But in general, agents often engage in a competitive setting, and revealing true procurement values can negatively impact them. For instance, knowing that a company has been procuring less than their history can reveal their strategic plans, devalue their market capital, hinder negotiations etc. We give a formalisation of the notion of privacy used in our setting in the next subsection.

\begin{figure}[t]
    \centering
    \includegraphics[width=0.5\textwidth]{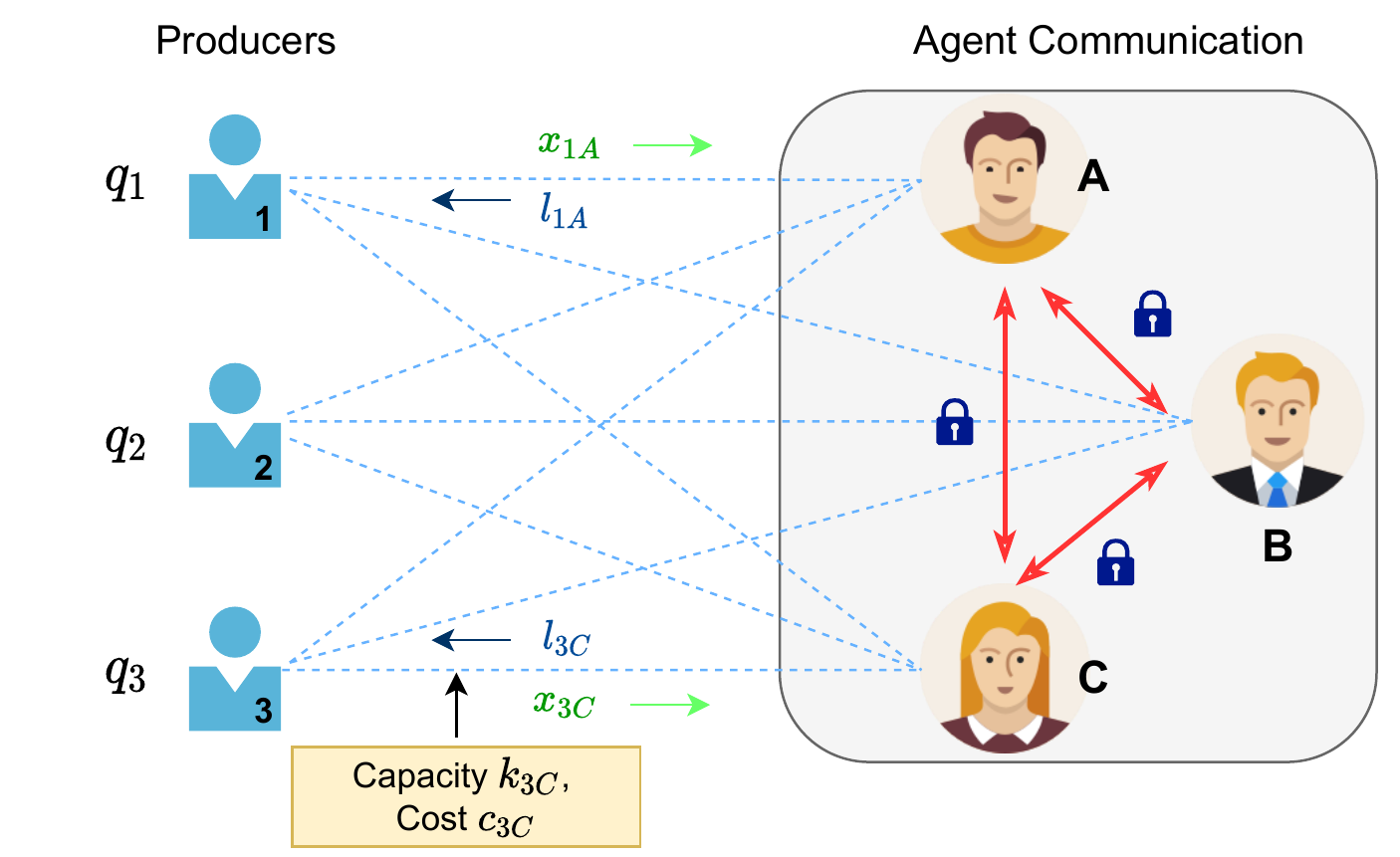}
    \caption{Overview of the communication model for \ouralgo: Agents interact with producers as part of the exploration and exploitation process. Agents also communicate among themselves to learn the qualities of producers. However, they share noisy data to maintain the privacy of their sensitive information.}
    \label{fig:my_model}
\end{figure}

\subsection{Differential Privacy (DP)} \label{sub:DP}
As opposed to typical federated models, we assume that the agents in our setting may be competing. Thus, agents will prefer the preservation of their sensitive information. Specifically, consider the history of procurement quantities $\mathbf{H}_{ij} = (l_{ij}^{t})_{t \in [T]}$ for any producer $i \in [m]$ is private to agent $j$. To preserve the privacy of $\mathbf{H}_{ij}$ while having meaningful utilitarian gains, we use the concept of Differential Privacy (DP). We tweak the standard DP definition in~\cite{dwork06,Dwork} for our setting. For this, let $\mathbf{S}_{j} = (\mathbf{s}_{j}^{t})_{t \in [T]}$ be complete history of procurement vectors for agent $j$.
\begin{definition}[Differential Privacy] \label{FDP}
     In a federated setting with $n \geq 2$ agents, a combinatorial MAB algorithm $A = (A_{j})_{j=1}^{n}$ is said to be $(\epsilon,\delta,n)-$differentially private if for any $u,v \in [n], s.t., u \neq v$, any $t_{o}$, any set of adjacent histories $\mathbf{H}_{iu} = (l_{iu}^{t})_{t \in [T]}, \mathbf{H}_{iu}^{'} = (l_{iu}^{t})_{t \in [T] \setminus \{t_{o}\}} \ \cup \ \Bar{l}_{iu}^{t_{o}}$ for producer $i$ and any complete history of procurement vector $\mathbf{S}_{v}$, 
    \begin{align*}
        \Pr(A_{v}(\mathbf{H}_{iu}) \in \mathbf{S}_{v}) \leq e^{\epsilon}\Pr(A_{v}(\mathbf{H}_{iu}^{'}) \in \mathbf{S}_{v}) + \delta 
    \end{align*}
\end{definition}

Our concept of DP in a federated CMAB formalizes the idea that the selection of procurement vectors by an agent is insusceptible to any single element $l_{ij}^{t}$ from another agent's procurement history. Note that the agents are not insusceptible to their own histories here.

Typically, the ``$\epsilon$" parameter is referred to as the \emph{privacy budget}. The \emph{privacy loss} variable $\mathcal{L}$ is often useful for the analysis of DP. More formally, given a randomised mechanism $\mathcal{M}(\cdot)$ and for any output $o$, the privacy loss variable is defined as,
 
 \begin{equation}\label{eqn::PL}
     \mathcal{L}^o_{\mathcal{M}(\mathbf{H})||\mathcal{M}(\mathbf{H}')} = \ln\left(\frac{\Pr[\mathcal{M}(\mathbf{H})=o]}{\Pr[\mathcal{M}(\mathbf{H}')=o]}\right).
 \end{equation}

\smallskip
\noindent\textit{Gaussian Noise Mechanism}~\cite{Dwork}.
To ensure DP, often standard techniques of adding noise to values to be communicated are used. The Gaussian Noise mechanism is a popular mechanism for the same. Formally, a randomised mechanism $\mathcal{M}(x)$ satisfies $(\epsilon,\delta)$-DP if the agent communicates $\mathcal{M}(x) \triangleq x + \mathcal{N}\left(0,\frac{2 \Delta(x)^{2} \ln(1.25 / \delta)}{\epsilon^{2}}\right) $. Here, $x$ is the private value to be communicated with \textit{sensitivity} $\Delta(x)$, and $\mathcal{N}(0,\sigma^2)$ the Gaussian distribution with mean zero and variance $\sigma^2$.


In summary, Figure~\ref{fig:my_model} provides an overview of the model considered. Recall that we aim to design a differentially private algorithm for federated CMAB with assured qualities. Before this, we first highlight the improvement in regret using the federated learning paradigm. Next, we discuss our private algorithm, \ouralgo, in Section \ref{sec:priv}.


\section{Non-private Federated Combinatorial Multi-armed Bandits} \label{sec:np}
\input{Ratio_gaussian}
We now demonstrate the advantage of federated learning in CMAB by highlighting the reduction in regret incurred compared to agents learning individually. We first categorize Federated CMAB into the following two settings: (i) \textit{homogeneous}: where the capacities and costs for producers are the same across agents, and (ii) \textit{heterogeneous}: where the producer's capacity and cost varies depending on the agent interacting with them.

\smallskip
\noindent\textbf{Homogeneous Setting.}
The core idea for single-agent learning in CMAB involves using standard $UCB$ exploration~\cite{Auer-2002}. We consider an Oracle that uses the $UCB$ estimates to return an optimal selection subset. In this paper, we propose that to accelerate the learning process and for getting \textit{tighter} error bound for quality estimations, the agents communicate their observations with each other in every round. In a homogeneous setting, this allows all agents to train a shared model locally without a central planner since the Oracle algorithm is considered deterministic. It's important to note that in such a setting, each agent has the same procurement history and the same expected regret.

Further, the quality constraint guarantees for the federated case follow trivially from the single agent case \cite[Theorem 2]{ayush}. Additionally, in Theorem~\ref{theorem:Homogenous Regret}, we prove that the upper bound for regret incurred by each agent is $\mathcal{O}(\frac{\ln (nT)}{n})$; a significant improvement over $\mathcal{O}(\ln T)$ regret the agent will incur when playing individually. \\

The proof entails breaking down the regret into three terms: the regret accumulated during the initial $\tau$ rounds, the regret caused by failing to meet quality constraints, and the regret resulting from sub-optimal arm selection. While the first two terms can be easily bounded, we address the last term by establishing an upper bound on the number of rounds in which sub-optimal arms are pulled, with a high probability. \\

\begin{algorithm} 
\small
    \caption{FCB} \label{Multiplayer SS-UCB}
        \begin{algorithmic}[1]
            \State \textbf{Inputs :} Total rounds $T$, Quality threshold $\alpha$,  $\epsilon$, $\delta$, Cost set $\{\mathbf{c}\}=\{(c_{i})_{i\in [m]}\}$, Capacity set $\{\mathbf{k}\}=\{(k_{i})_{i\in [m]}\}$,
            \State $\forall j \in [n]$ Initialise $W_{i}$ (Total units procured from producer $i$) and $\hat{q}_{i}$ (quality estimate for producer $i$)
            \While{$t \leq \frac{3 ln (yt)}{2 n \epsilon_2^2}$ \textbf{(Explore Phase)}}
                \For{ each agent $j \in [n]$}
                    \State Pick a procurement vector $\mathbf{s}^{t} = (1)^{m}$
                    \State Observe quality realisations $\mathbf{X}_{\mathbf{s}^{t},j}^{t}$
                    \State (\textbf{Synchronise}) Communicate $\mathbf{X}_{\mathbf{s}^{t},j}^{t}$ to all other agents
                    \State $[\forall i \in [m]]$ $\hat{q}_{i} \longleftarrow \frac{\hat{q}_{i}W_{i} + \sum_{j \in [n]} x_{ij}^{t}}{W_{i} + n}$
                    \State $[\forall i \in [m]]$ $W_{i} \longleftarrow W_{i} + n$
                \EndFor
                \State $t \longleftarrow t+1$
            \EndWhile
            \While{$t \leq T$ \textbf{(Explore-Exploit Phase)}}
                \For{ each agent $j \in [n]$}
                    \State $[\forall i \in [m]]$ $(\hat{q}_{i})^{+} = \hat{q}_{i} + \sqrt{\frac{3 ln (nt)}{2nW_{i}}}$
                    \State Pick a procurement vector $\mathbf{s}^{t} = Oracle(\{(\hat{q}_{i})^{+}\}_{i \in [m]}, $
                    \State $\mathbf{c},,\alpha + \gamma, R)$
                    \State Observe quality realisations $\mathbf{X}_{\mathbf{s}^{t},j}^{t}$
                    \State (\textbf{Synchronise}) Communicate $\mathbf{X}_{\mathbf{s}^{t},j}^{t}$ to all other agents
                    \State $[\forall i \in [m]]$ $\hat{q}_{i} \longleftarrow \frac{\hat{q}_{i}W_{i} + \sum_{j \in [n]} x_{ij}^{t}}{W_{i} + n l_{i}}$
                    \State $[\forall i \in [m]]$ $W_{i} \longleftarrow W_{i} + n l_{i}$
                \EndFor
            \EndWhile
        \end{algorithmic}
        \label{m-a SS-UCB}
\end{algorithm}

\begin{theorem} \label{theorem:Homogenous Regret}
    For Federated CMAB in a homogeneous setting with $n$ agents, if the qualities of producers satisfy $\gamma$-seperatedness, then the individual regret incurred by each of the agents is bounded by $\mathcal{O}(\frac{\ln (nT)}{n})$.
\end{theorem}
\begin{proof}
For the rest of the proof, we consider any arbitrary agent $j \in [n]$ and omit explicit denotation. \\
\begin{align*}
    \mathcal{R}^{t}_{A} & = \sum_{t=1}^{\tau - 1} \mathcal{R}^{t}_{A} + \sum_{t=\tau}^{T} \mathcal{R}^{t}_{A}  \\
    & \leq L\tau + \sum_{t=\tau}^{T} \mathcal{R}^t_{A}
\end{align*}
\begin{align*}
    E[\mathcal{R}^t_{A}] & \leq L\tau + \sum_{t \geq \tau} [(1-\sigma)(r_{\mathbf{s}^{*}} - \\
    & r_{\mathbf{s}_{A}^{t}}) + \sigma L]
\end{align*}
Let,
\begin{align*}
    \mathcal{R}_{u}^{T}  = \sum_{t \leq \tau} (1-\sigma)(r_{\mathbf{s}^{*}} - r_{\mathbf{s}_{A}^{t}})
\end{align*}
Here, $\sigma$ is the probability with which quality constraint is satisfied.
Some Additional Notations: \\
\begin{enumerate}
    \item $V^T$ : Number of times a sub-optimal procurement vector is chosen.
    \item $F^t$ : Event that Oracle failed to produce $\omega$-approximation solution.
    \item $W_{i}^{t}$ : Total units procured from $i$ till round $t$. 
    \item $S_b$ denotes the set of bad procurement vectors.
    \item $k = argmax_{i \in [m]}\ k_i$, represents the max capacity amongst all arms.
    \item $\Delta_{min}^i = \omega r_{\mathbf{s}^{*}} - \max \{r_{\mathbf{s}}, \mathbf{s} \in S_b, l_{i} \neq 0 \}$.
    \item $\Delta_{min} =\min_{i \in [m]} \Delta_{min}^i$
    \item $\Delta_{max}^i = \omega r_{\mathbf{s}^{*}} - \min \{r_{\mathbf{s}}, \mathbf{s} \in S_b, l_{i} \neq 0 \}$.
\end{enumerate}
\smallskip
We can see that,
\begin{align}
    E[\mathcal{R}_{u}^{T}] \leq E[V_{T}] \Delta_{max}
\end{align}
\textbf{Bounding number of round in which sub-optimal procurement vector are chosen} \\
We can use a proof sketch similar to the proof provided in \cite{CUMB} to tightly bound $V^{T}$. Let each arm $i$ have a counter $Z_i$ associated with it. $Z_{i}^{t}$ represents the value of $Z_i$ after $t$ rounds. \\
Counters $\{Z_{i}\}_{i \in [m]}$ are updated as follows,
\begin{enumerate}
    \item After initial $m$ rounds, $\sum_i Z_i^m = m$.
    \item For round $t>m$, let $\mathbf{s}^t$ be the selected procurement vector in round $t$. We say round $t$ is bad if oracle selects a bad arm.
    \item For a bad round, we increase one of the counters. Let $j = \textit{argmin}_{i \in [m], l_{i}^{t} \neq 0}\ Z_{i}^{t-1}$, then $Z_{j}^{t} = Z_{j}^{t-1} + 1$ (If multiple counters have min value, select $i$ randomly from the set).
\end{enumerate}
Total number of bad rounds in first $p$ rounds is less than or equal to $\sum_i Z_{i}^{p}$. \\ \\
Let $\gamma_t  =\frac{6 \log(nt)}{n (f^{-}(\Delta_{min}))^2}$,
\begin{align}
    &\sum_{i=1}^{m} Z_{i}^{p} - m(\gamma_p + 1) \nonumber \\ 
    &= \sum_{t=m+1}^{p} \mathbb{I}\{\mathbf{s}^t \in S_b\} - m \gamma_{p} \nonumber \\ 
    &\leq \sum_{t=m+1}^{p} \sum_{i=1}^{m} \mathbb{I}\{\mathbf{s}^t \in S_b, Z_{i}^{t} > Z_{i}^{t-1}, Z_{i}^{t-1} > \gamma_p\} \nonumber \\
    &\leq \sum_{t=m+1}^{p} \sum_{i=1}^{m} \mathbb{I}\{\mathbf{s}^t \in S_b, Z_{i}^{t} > Z_{i}^{t-1}, Z_{i}^{t-1} > \gamma_t\}  \nonumber\\
    &= \sum_{t=m+1}^{p} \mathbb{I}\{\mathbf{s}^t \in S_b, \forall i\ s.t.\ l_{i}^{t} \neq 0, Z_{i}^{t-1} > \gamma_t\} \label{eq:prf 1} \\  
    &\leq \sum_{t=m+1}^{p} \mathbb{I}\{F^t\} + \mathbb{I}\{\neg F^t, \mathbf{s}^t \in S_b, \forall i\ s.t.\ l_{i}^{t} \neq 0, Z_{i}^{t-1} > \gamma_t\} \nonumber \\
    &\leq \sum_{t=m+1}^{p} \mathbb{I}\{F^t\} + \mathbb{I}\{\neg F^t, \mathbf{s}^t \in S_b, \forall i\ s.t.\ l_{i}^{t} \neq 0, W_{i}^{t-1} > \gamma_t\} \nonumber 
\end{align}
Eq. \eqref{eq:prf 1} holds due to the rule of updating the counters. \\ \\
Now we first claim that $Pr\{\neg F^t,\mathbf{s}^t \in S_b, \forall i\ s.t.\ l_{i}^{t} \neq 0, W_{i}^{t-1} > \gamma_t\} \leq 2 k n^{-3} t^{-2}$. \\ \\ 
For any $i \in [m]$,
\begin{align}
    &Pr[\lvert \hat{q}_{i,W_{i}^{t-1}} - q_{i} \rvert \geq \sqrt{\frac{3 \log(nt)}{2nW_{i}^{t-1}}}] \nonumber \\
    &= \sum_{b=1}^{k(t-1)} Pr[\lvert \hat{q}_{i,W_{i}^{t-1}} - q_{i} \rvert \geq \sqrt{\frac{3 \log(nt)}{2nb}}, b = W_{i}^{t-1}] \nonumber \\
    &\leq \sum_{b=1}^{k(t-1)} Pr[\lvert \hat{q}_{i,W_{i}^{t-1}} - q_{i} \rvert\geq \sqrt{\frac{3 \log(nt)}{2nb}}] \nonumber \\
    & \leq \sum_{b=1}^{k(t-1)} 2e^{-2(ns)(\frac{3 \log(nt)}{2nb})} \label{eq:prf 2} \\
    &= 2kn^{-3}t^{-2} \nonumber
\end{align}
Eq. \ref{eq:prf 2} holds due to Hoeffding inequality. \\

\noindent Let $\Lambda_{i}^{t} = \sqrt{\frac{3 \log(nt)}{2nW_{i}^{t-1}}}$. \\ 

\noindent Let $E^t = \{\forall i \in [m], \lvert \hat{q}_{i,W_{i}^{t-1}} - q_{i} \rvert\leq \Lambda_{i}^{t} \}$ be an event. Then by union bound on Eq. \ref{eq:prf 2}, $Pr[\neg E^t] \leq 2kn^{-3}t^{-2}$. Also, since $\lvert ({q}_{i}^{t})^{+} - \hat{q}_{i}^{t} \rvert = \Lambda_{i,t} $, that means, $E^t \implies ({q}_{i}^{t})^{+} \geq q_{i}^{t}, \forall i \in [m]$. \\ \\
Let $\Lambda = \sqrt{\frac{3 \log(nt)}{2n\gamma_t}}$ and $\Lambda^{t} = \max_{i \in [m]} \Lambda_{i}^{t}$.
\begin{align}
E^t &\implies \lvert ({q}_{i}^{t})^{+} - q_{i}^{t} \rvert \leq 2 \Lambda^{t} \label{eq:prf 3}
\end{align}
\begin{align}
    \{\mathbf{s}^t \in S_b, \forall i\ s.t.\ l_{i}^{t} \neq 0, W_{i}^{t-1} > \gamma_t \} &\implies \Lambda > \Lambda^t \label{eq:prf 4}
\end{align}
If $\{E^t,\neg F^t, \mathbf{s}^t \in S_b, \forall i\ s.t.\ l_{i}^{t} \neq 0, W_{i}^{t-1} > \gamma_t \}$ holds true, then using Eq. \ref{eq:prf 3}, Eq. \ref{eq:prf 4}, monotonicity of rewards and bounded smoothness property,
\begin{align}
    r_{\mathbf{s}^{t}} + f(2\Lambda) > \omega r_{\mathbf{s}^{*}} \label{eq:prf 5}
\end{align}

\noindent Since $\gamma_t  =\frac{6 \log(nt)}{n (f^{-}(\Delta_{min}))^2}$, $f(2\Lambda) = \Delta_{min}$. This is contradictory to definition of $\Delta_{min}$.
\begin{align}
    &Pr\{E^t,\neg F^t, \mathbf{s}^t \in S_b, \forall i\ s.t.\ l_{i}^{t} \neq 0, W_{i}^{t-1} > \gamma_t\} = 0 \nonumber \\
    &\implies \{\neg F^t, \mathbf{s}^t \in S_b, \forall i\ s.t.\ l_{i}^{t} \neq 0, W_{i}^{t-1} > \gamma_t\} \nonumber\\
    &\leq Pr[\neg E^t] \leq 2kn^{-3}t^{-2} \nonumber
\end{align}
Thus,
\begin{align}
    &E[\sum_i^m Z_{i}^{p}] \leq m(\gamma_p + 1) + (1-\beta)(p-m) + \sum_{t=1}^{p}2mkn^{-3}t^{-2} \nonumber \\
    &\leq \frac{6m\log(nt)}{n(f^{-1}(\Delta_{min}))^2} + (\frac{\pi^2}{3} + 1)mkn^{-3} + (1-\beta)(p-m) \label{eq:prf 6}
\end{align}
\textbf{Bounding Regret}
Using Eq. \ref{eq:prf 6} and using the fact that $\beta = 0$
\begin{align*}
    E[\mathcal{R}_{u}^{T}] &= E[V^T] \Delta_{max} \\
    &\leq \left( \frac{\pi^2}{3}kn^{-3} + \frac{6\log(nt)}{n(\frac{\Delta_{min}}{R})^2}\right) m\Delta_{max}
\end{align*}

This completes our proof for regret bound in a homogeneous federated setting.

\end{proof}

\smallskip
\noindent\textbf{Heterogeneous Setting.} 
In real-world, the agents may not always have the same capacities. For such a heterogeneous setting, the regret analysis is analytically challenging. For instance, we can no longer directly use Hoeffding's inequality, needed for proving Theorem~\ref{theorem:Homogenous Regret}, since the procurement histories will differ across agents. Still, the intuition for regret reduction from cooperative learning carries over.

Even in a heterogeneous setting, communicating the observations allows the agent to converge their quality estimations to the mean faster and provide tighter error bounds. Even with shared quality estimates, Oracle may return different procurement vectors for different agents based on different capacities. Thus, a weighted update in estimation is essential, and the procurement vector would also need to be communicated.

\smallskip
We empirically demonstrate that using federated learning in heterogeneous setting shows similar $FRR$ (ratio of regret incurred in federated setting compared to non federated setting) trend compared to homogeneous setting, over $100000$ rounds for two scenarios: (i) Costs and qualities are sampled from uniform distributions, i.e. $c_{ij} \sim U[0,1]$, $q_{i} \sim U[0,1]$, (ii) Costs and qualities are sampled from normal distributions around the quality threshold, i.e., $c_{ij} \sim \mathcal{N}(\alpha,0.1)$, $q_{i} \sim \mathcal{N}(\alpha,0.1)$. 

Fig.~\ref{fig:ratio} depicts the results. From Fig.~\ref{fig:ratio} we observe that the trend for both homogeneous and heterogeneous settings are quite similar. This shows that, similar to the homogeneous setting, employing federated learning reduces regret even in the heterogeneous setting.



\section{\ouralgo: Privacy-preserving Federated Combinatorial Bandit} \label{sec:priv}
From Section~\ref{sub:DP}, recall that we identify the procurement history of an agent-producer pair as the agent's sensitive information. We believe that the notion of DP w.r.t. the agent-producer procurement history is reasonable. A differentially private solution ensures that the probability with which other agents can distinguish between an agent's adjacent procurement histories is upper bounded by the privacy budget $\epsilon$.

\smallskip
\noindent\underline{Section Outline}: In this section, we first argue that naive approaches for DP are not suitable due to their lack of meaningful privacy guarantees. Second, we show that all attributes dependent on the sensitive attribute must be sanitised before sharing to preserve privacy. Third, we define a privacy budget algorithm scheme. Fourth, we formally introduce \ouralgo\ including a selective learning procedure. Last, we provide the $(\epsilon,\delta)$-DP guarantees for \ouralgo.

\subsection{Privacy budget and Regret Trade-off}\label{sub:tradeoff}
Additive noise mechanism (e.g., Gaussian Noise mechanism~\cite{Dwork}) is a popular technique for ensuring $(\epsilon,\delta)$-DP. To protect the privacy of an agent's procurement history within the DP framework, we can build a naive algorithm for heterogeneous federated CMAB setting by adding noise to the elements of the procurement vectors being communicated in each round. 

However, such a naive approach does not suitably satisfy our privacy needs. Using the Basic Composition theorem \cite{Dwork}, which adds the $\epsilon$s and $\delta$s across queries, it is intuitive to see that communicating in every round results in a high overall $\epsilon$ value which may not render much privacy protection in practice~\cite{triastcyn2019federated}. Consider the agents interacting with the producers for $10^6$ rounds. Let $\epsilon = 10^{-2}$ for each round they communicate the perturbed values. Using Basic Composition, we can see that the overall privacy budget will be bounded by $\epsilon = 10^4$, which is practically not acceptable. The privacy loss in terms of overall $\epsilon$ grows at worst linearly with the number of rounds. 

It is also infeasible to solve this problem merely by adding more noise (reducing $\epsilon$ per round) since if the communicated values are too noisy, they can negatively affect the estimates. This will result in the overall regret increasing to a degree that it may be better to not cooperatively learn. To overcome this challenge, we propose to decrease the number of rounds in which agents communicate information. 

Secondly, if the sample size for the local estimates is too small, noise addition can negatively effect the regret incurred. On the other hand, if the sample size of local estimate is too large, the local estimate will have tight error bounds and deviating from the local estimate too much may result in the same. 

\smallskip
\noindent{\textbf{When to Learn.}} Based on the above observations, we propose the following techniques to strike an effective trade-off between the privacy budget and regret.
\begin{enumerate}
    \item To limit the growth of $\epsilon$ over rounds, we propose that communication happens only when the current round number is equal to a certain threshold (denoted by $\tau$) which doubles in each communication round. Thus, there are only $\log(T)$ communications rounds, where density of communication rounds decrease over rounds.
    \item We propose to communicate only for a specific interval of rounds, i.e., for each round $t \in [\underline{t},\Bar{t}]$. \textit{No} communication occurs outside these rounds. This ensures that agent communication only happens in rounds when it is useful and not detrimental.

\end{enumerate}

\subsection{Additional Information Leak with Actual Quality Estimates and Noisy Weights}\label{sub:infoLeak}
It is also important to carefully evaluate the way data is communicated every round since it may lead to privacy leaks. For example, consider that all agents communicate their local estimates of the producer qualities and perturbation of the total number of units procured from each producer to arrive at the estimation.  We now formally analyse the additional information leak in this case. W.l.o.g. our analysis is for any arbitrarily picked producer $i \in [m]$ and agent $j \in [n]$. As such, we omit the subscripts ``$i$" for producer and ``$j$" for the agent. We first set up the required notations as follows.

\smallskip
\noindent\underline{Notations}: Consider $\hat{q}^t, W^t$ as \textit{true} values for the empirical estimate of quality and total quantity procured till the round $t$ (not including $t$). Next, let $\Tilde{W}^t$ denote \textit{noisy} value of $W^t$ (with the noise added using any additive noise mechanism for DP~\cite{Dwork}). We have $w^t$ as the quantity procured in round $t$. Last, let $\hat{q}^{obsv_t}$ denote the quality estimate based on just round $t$. Through these notations, we can compute $\hat{q}^{t+1}$ for the successive round $t+1$ as follows: $\hat{q}^{t+1} = \frac{W^t \times \hat{q}^t + w^t \times \hat{q}^{obsv_t}}{W^t + w^t}$.

\begin{Claim}\label{claim}
Given  $\hat{q}^t, W^t, \Tilde{W}^t, w^t$ and $\hat{q}^{obsv_t}$, the privacy loss variable $\mathcal{L}$ is not defined if $\hat{q}^t$ is also not perturbed.
\end{Claim}
\begin{proof}
If $w^t = 0$, then it follows that $\hat{q}^{t+1} = \hat{q}^t$ irrespective of $\Tilde{W}^t, \Tilde{W}^{t+1}$. So, if it values $\hat{q}^{t+1} \neq \hat{q}^t$ are communicated, other agents can conclude that  $w^{t}$ cannot be zero. This implies that the privacy loss variable $\mathcal{L}$ (Eq.~\ref{eqn::PL}) is not defined as an adversary can distinguish between two procurement histories.  
\end{proof}

    

With Claim~\ref{claim}, we show that $\epsilon$ may not be bounded even after sanitising the sensitive data due to its dependence on other non-private communicated data. This is due to the fact that the local mean estimates are a function of the procurement vectors and the observation vectors. Thus, it becomes insufficient to just perturb the quality estimates. 

We propose that whenever communication happens, only procurement and observation values based on rounds since last communication are shared. Additionally, to communicate weighted quality estimates, we use the Gaussian Noise mechanism to add noise to \textit{both} the procurement values and realisation values. The sensitivity ($\Delta$) for noise sampling is equal to the capacity of the producer-agent pair.


\begin{algorithm}
\small

\setcounter{algorithm}{0}
    \floatname{algorithm}{Procedure}
    \caption{\textsf{CheckandUpdate}($W, \Tilde{w}, Y, \Tilde{y}, \omega_{1}, \omega_{2}, n, t$)}
    \label{procedure}
    \begin{algorithmic}[1]
        \State $\hat{q} \longleftarrow \frac{Y}{W}$ 
        \If{ $\frac{\Tilde{y}}{\Tilde{w}} \in \left[\hat{q} - \omega_1 \sqrt{\frac{3 ln (nt)}{2W}}, \hat{q} + \omega_1 \sqrt{\frac{3 ln (nt)}{2W}}\right]$}
            \State $W \longleftarrow W + \omega_2 \Tilde{w}$
            \State $Y \longleftarrow Y + \omega_2 \Tilde{y}$
        \EndIf
        \State \textbf{return} $W,Y$
    \end{algorithmic}
\end{algorithm}

\subsection{Privacy Budget Allocation}
Since the estimates are more sensitive to noise addition when the sample size is smaller, we propose using monotonically decreasing privacy budget for noise generation. Formally, let total privacy budget be denoted by $\epsilon$ with $(\epsilon^{1},\epsilon^{2},\ldots)$ corresponding to privacy budgets for communication rounds $(1,2,\ldots)$. Then, we have $\epsilon^{1} > \epsilon^{2} > \ldots$. Specifically, we denote $\epsilon^{z}$ as the privacy budget in the $z^{th}$ communication round, where $
    \epsilon^z \longleftarrow \frac{\epsilon}{2 \times \log(T)} + \frac{\epsilon}{2^{z+1}}$.

\begin{algorithm}[!ht]
\small
\setcounter{algorithm}{1}
        \caption{\ouralgo} \label{algo:Multiplayer Private SS-UCB}
        \begin{algorithmic}[1]
            \State \textbf{Inputs :} Total rounds $T$, Quality threshold $\alpha$, $\epsilon$, $\delta$, Cost set $\{\mathbf{c}_{j}\}=\{(c_{i,j})_{i\in [m]}\}$, Capacity set $\{\mathbf{k}_{j}\}=\{(k_{i,j})_{i\in [m]}\}$, Start round $\underline{t}$, Stop round $\overline{t}$
            \LineComment{Initialisation Step}
            \State $t \longleftarrow 0$, $\tau \longleftarrow 1$
            \State $[\forall i \in [m],\forall j \in [n]]$ Initialise total and uncommunicated procurement ($W_{i,j}, w_{i,j}$) and realisations ($Y_{i,j}, y_{i,j}$)
            \While{$t \leq \frac{3 ln (yT)}{2 n \zeta^2}$ \textbf{(Pure Explore Phase)}}
                \For{all the agents $j \in [n]$}
                    \State Pick procurement vector $\mathbf{s}_{j}^{t} = (1)^{m}$ and observe quality realisations $\mathbf{X}_{\mathbf{s}_{j}^{t},j}^{t}$.
                    \State $[\forall i \in [m]]$ Update $W_{i,j}^{t+1}, w_{i,j}^{t+1}, Y_{i,j}^{t+1}, y_{i,j}^{t+1}$ using Eq.~\ref{eq:Y} 
                    \If{ $t \in [\underline{t}, \overline{t}]$ and $t \geq \tau$} \Comment{\textcolor{blue}{Communication round}}
                        \State $[\forall i \in [m]]$ Calculate $\Tilde{w}_{i,j}, \Tilde{y_{i,j}}$ according to Eq. \ref{w private},\ref{y private}
                        \For{each agent $z \in [n] / j$}
                            \State Send $\{ \Tilde{w}_{i,j} ,  \Tilde{y}_{i,j}\}$ to agent $z$
                            \State $[\forall i \in [m]]$ $W_{i,z}^{t+1}, Y_{i,z}^{t+1} \longleftarrow$ \textsf{CheckandUpdate}($W_{i,z}^{t+1}, \Tilde{w}_{i,j}, Y_{i,z}^{t+1}, \Tilde{y}_{i,j},.$) 
                        \EndFor
                        \State $[\forall i \in [m]]$ $w_{i,j}^{t+1} \longleftarrow 0$, $y_{i,j}^{t+1} \longleftarrow 0$
                        \State $\tau \longleftarrow 2 \times \tau$
                    \EndIf
                    \State Update quality estimate
                    \State $t \longleftarrow t+1$
                \EndFor
            \EndWhile
            \While{$t \leq T$, $\forall j \in [n]$ \textbf{(Explore-Exploit Phase)}}
                \State $[\forall i \in [m]]$ Calculate the upper confidence bound of quality estimate, $(\hat{q}_{i,j}^{t}) ^{+}$
                \State Pick procurement vector using $\mathbf{s}_{j}^{t} = \mathbf{Oracle(}(\hat{q}_{i,j}^{t}) ^{+},\mathbf{c}_{j},\mathbf{k}_{j},.\mathbf{)}$ and observe its realisations $\mathbf{X}_{\mathbf{s}_{j}^{t},j}^{t}$.
                \State $[\forall i \in [m]]$ Update $W_{i,j}^{t+1}, w_{i,j}^{t+1}, Y_{i,j}^{t+1}, y_{i,j}^{t+1}$ using Eq.~\ref{eq:Y} 
                    \If{ $t \in [\underline{t}, \overline{t}]$ and $t \geq \tau$} \Comment{\textcolor{blue}{Communication round}}
                        \State $[\forall i \in [m]]$ Calculate $\Tilde{w}_{i,j}, \Tilde{y_{i,j}}$ according to Eq. \ref{w private},\ref{y private}
                        \For{each agent $z \in [n] / j$}
                            \State Send $\{ \Tilde{w}_{i,j} ,  \Tilde{y}_{i,j}\}$ to agent $z$
                            \State $[\forall i \in [m]]$ $W_{i,z}^{t+1}, Y_{i,z}^{t+1} \longleftarrow$ \textsf{CheckandUpdate}($W_{i,z}^{t+1}, \Tilde{w}_{i,j}, Y_{i,z}^{t+1}, \Tilde{y}_{i,j},.$) 
                        \EndFor
                        \State $[\forall i \in [m]]$ $w_{i,j}^{t+1} \longleftarrow 0$, $y_{i,j}^{t+1} \longleftarrow 0$
                        \State $\tau \longleftarrow 2 \times \tau$
                    \EndIf
                \State Update quality estimate
                \State $t \longleftarrow t+1$
            \EndWhile
        \end{algorithmic}
\end{algorithm}


\subsection{\ouralgo: Algorithm}
Based on the feedback from the analysis made in previous subsections, we now present a private federated CMAB algorithm for the heterogeneous setting, namely \ouralgo. Algorithm~\ref{algo:Multiplayer Private SS-UCB} formally presents \ouralgo. Details follow. 

\smallskip
\noindent\textbf{Algorithm~\ref{algo:Multiplayer Private SS-UCB} Outline.} The rounds are split into two phases. During the initial pure exploration phase (Lines 6-22), the agents explore all the producers by procuring evenly from all of them. The length of the pure exploration phase is carried over from the non-private algorithm. In this second phase (Lines 23-38), explore-exploit, the agents calculate the $UCB$ for their quality estimates. Then the Oracle is used to provide a procurement vector based on the cost, capacity, $UCB$ values as well as the quality constraint ($\alpha$). Additionally, the agents communicate their estimates as outlined in Sections~\ref{sub:tradeoff} and \ref{sub:infoLeak}.  The agents update their quality estimates at the end of each round using procurement and observation values (both local and communicated), Lines 19 and 36. 
\begin{equation}\label{eq:Y}
\begin{gathered}
        w_{i,j}^{t+1} \longleftarrow w_{i,j}^{t} + l_{i,j}^{t} ~;~ W_{i,j}^{t+1} \longleftarrow W_{i,j}^{t} + l_{i,j}^{t} \\
        y_{i,j}^{t+1} \longleftarrow y_{i,j}^{t} + x_{i,j}^{t} ~;~ Y_{i,j}^{t+1} \longleftarrow Y_{i,j}^{t} + x_{i,j}^{t} \\
    q_{i,j}^{t+1} \longleftarrow \frac{Y_{i,j}^{t+1}}{W_{i,j}^{t+1}}
\end{gathered}
\end{equation}


\smallskip
\noindent\textbf{Noise Addition.}
From Section \ref{sub:infoLeak}, we perturb both uncommunicated procurement and realization values for each agent-producer pair using the Gaussian Noise mechanism. Formally, let $w_{i,j}^{t}, y_{i,j}^{t}$ be the uncommunicated procurement and realization values. Then $\Tilde{w}_{i,j}, \Tilde{y}_{i,j}$ are communicated, which are calculated using the following privatizer,

\begin{align}
    \Tilde{w}_{i,j} = w_{i,j}^{t} + \mathcal{N}(0,\frac{2 k_{i,j}^{2} \log(1.25 / \delta)}{(\epsilon^{z})^{2}}) \label{w private} \\
    \Tilde{y}_{i,j} = y_{i,j}^{t} + \mathcal{N}(0,\frac{2 k_{i,j}^{2} \log(1.25 / \delta)}{(\epsilon^{z})^{2}}) \label{y private}
\end{align}

where $\epsilon^{z}$ is the privacy budget corresponding to the $z^{th}$ communication round.

\smallskip
\noindent\textbf{What to Learn.} To minimise the regret incurred, we propose that the agents selectively choose what communications to learn from. Weighted confidence bounds around local estimates are used to determine if a communication round should be learned from. Let $\xi_{i,j}^{t} = \sqrt{\frac{3 ln(t)}{2\sum_{z \in \{1,2,\ldots,t\}} l_{i,j}^{z} }}$ denote the confidence interval agent $j$ has w.r.t. local quality estimate of producer $i$. Then, the agents only selects to learn from a communication if $\hat{q}_{i,j}^{t} - \omega_{1} \xi_{i,j}^{t} < q_{(communicated)i,j} < \hat{q}_{i,j}^{t} + \omega_{1} \xi_{i,j}^{t}$ where $\omega_{1}$ is a weight factor and $q_{(communicated)i,j} = \frac{\Tilde{y}_{i,j}}{\Tilde{w}_{i,j}}$.

The local observations are weighed more compared to communicated observations for calculating overall estimates. Specifically, $\omega_{2} \in [0,1]$ is taken as the weighing factor for communicated observations.


\subsection{\ouralgo: $(\epsilon,\delta)$-DP Guarantees}
In each round, we perturb the values being communicated by adding Gaussian noises satisfying $(\epsilon',\delta')$-DP to them. It is a standard practice for providing DP guarantees for group sum queries. Let $\mathcal{M}$ be a randomised mechanism which outputs the sum of values for a database input $d$ using Gaussian noise addition. Since Oracle is deterministic, each communication round can be considered a post-processing of $\mathcal{M}$ whereby subset of procurement history is the the database input. Thus making individual communication rounds satisfy $(\epsilon',\delta')$-DP. 

The distinct subset of procurement histories used in each communication round can be considered as independent DP mechanisms. Using the Basic Composition theorem, we can compute the overall $(\epsilon,\delta)$-DP guarantee. In \ouralgo, we use a target privacy budget, $\epsilon$, to determine the noise parameter $\sigma$ in each round based on Basic composition. Thus, this can be leveraged as a tuning parameter for privacy/regret optimisation. 
\section{Experimental Results}\label{sec:exp}
In this section, we compare \ouralgo\ with non-federated and non-private approaches for the combinatorial bandit (CMAB) setting with constraints. We first explain the experimental setup, then note our observations and analyze the results obtained. 
\input{T}
    \begin{figure}[!ht]
        \begin{subfigure}[b]{0.475\textwidth}
          \centering
        \begin{tikzpicture}
        \begin{axis}[
            width=\textwidth,
            height = 5cm,
            tick scale binop=\times,
            title={\ouralgo: FRR vs. $\epsilon$},
            ytick={0,0.5,1,1.5},
            xlabel={Privacy Budget ($\epsilon$)},
            xlabel near ticks,
            ylabel={FRR},
            ylabel near ticks,
            ylabel near ticks, yticklabel pos=left, 
            ymajorgrids=true,
            grid style=dashed,
            legend style={at={(1,1)},{column sep=0.25cm}},  
            legend entries={\ouralgo,\texttt{FCB}},
            every axis plot/.append style={thick}
        ]
\addplot[
    color=blue,
dotted, every mark/.append style={solid, fill=gray}, mark=square*,mark repeat = 1]
coordinates {
(0.2, 1.71)
(0.3, 1.14)
(0.4, 1.02)
(0.5, 0.95)
(0.6, 0.87)
(0.7, 0.84)
(0.8, 0.79)
(0.9, 0.78)
(1, 0.78)
(1.1, 0.74)};



\end{axis}
\end{tikzpicture}
\label{fig:epsilon_uniform}
\caption{$c{_{ij}},q_{i} \sim U[0,1]$}
\end{subfigure}
\begin{subfigure}[b]{0.475\textwidth}
      \centering
        \begin{tikzpicture}
        \begin{axis}[
            width=\textwidth,
            height = 5cm,
            tick scale binop=\times,
            title={\ouralgo: FRR vs. $\epsilon$},
            ytick={0,0.5,1,1.5},
            xlabel={Privacy Budget ($\epsilon$)},
            xlabel near ticks,
            ylabel={FRR},
            ylabel near ticks,
            ylabel near ticks, yticklabel pos=left, 
            ymajorgrids=true,
            grid style=dashed,
            legend style={at={(1,1)},{column sep=0.25cm}},
            legend entries={\ouralgo,\texttt{FCB}},
            every axis plot/.append style={thick}
        ]

\addplot[
    color=blue,
dotted, every mark/.append style={solid, fill=gray}, mark=square*,mark repeat = 1]
coordinates {
(0.2, 1.3)
(0.3, 0.83)
(0.4, 0.77)
(0.5, 0.70)
(0.6, 0.662)
(0.7, 0.666)
(0.8, 0.660)
(0.9, 0.665)
(1,  0.678)
(1.1, 0.668)};


\end{axis}
\end{tikzpicture}
\label{fig:epsilon_gaussian}
\caption{$c{_{ij}},q_{i} \sim \mathcal{N}(\alpha=0.4,0.2)$ }
\end{subfigure}
    \caption{\texttt{EXP2}: FRR for \ouralgo\ while varying privacy budget $\epsilon$  (with $n=10$, $m=30$, $t=100000$)}
    \label{fig:exp2}
\end{figure}
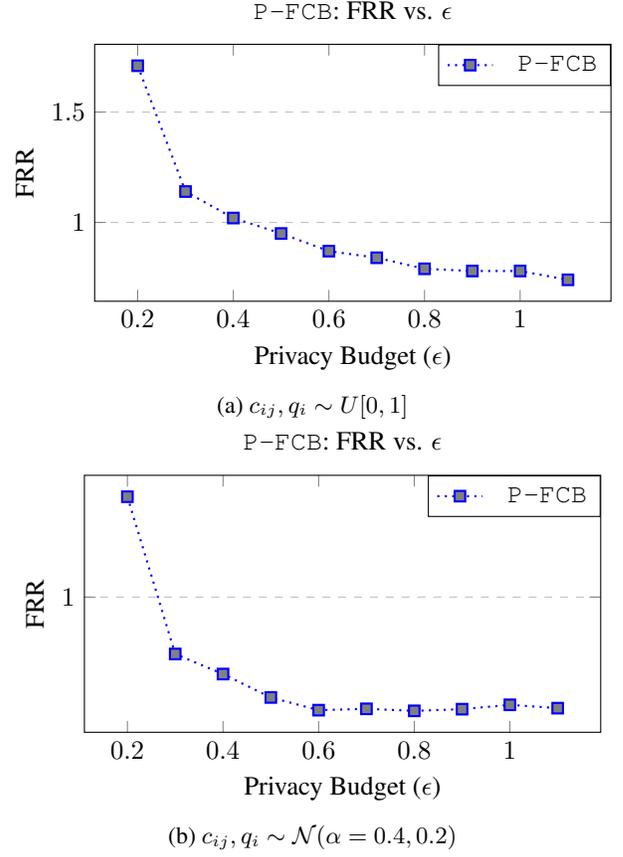
\begin{figure}[t]
    \begin{subfigure}[b]{0.475\textwidth}
          \centering
        \begin{tikzpicture}
        \begin{axis}[
            width=\textwidth,
            height = 5cm,
            title={\ouralgo: Regret vs. $n$},
            xlabel={$\#$ of Agents ($n$)},
            xlabel near ticks,
            ylabel={Average regret per agent},
            ylabel near ticks,
            ylabel near ticks, yticklabel pos=left, 
            ymajorgrids=true,
            grid style=dashed,
            every axis plot/.append style={ thick},
            ymax = 3000,
            legend style={at={(1,1)},{column sep=0.25cm}},  
            legend entries={\ouralgo\ ($\epsilon=1$)}
        ]

\addplot[
    color=purple,
dotted, every mark/.append style={solid, fill=gray}, mark=triangle*,mark repeat = 1]
coordinates {
(10,1976)
(20,1418)
(30,1293)
(40,1040)
};
\end{axis}
\end{tikzpicture}
\caption{$c{_{ij}},q_{i} \sim U[0,1]$}
    \end{subfigure}
    \begin{subfigure}[b]{0.475\textwidth}
          \centering
        \begin{tikzpicture}
        \begin{axis}[
            width=\textwidth,
            height = 5cm,
            title={\ouralgo: Regret vs. $n$},
            xlabel={$\#$ of Agents ($n$)},
            xlabel near ticks,
            ylabel={Average regret per agent},
            ylabel near ticks,
            ylabel near ticks, yticklabel pos=left, 
            ymajorgrids=true,
            grid style=dashed,
            every axis plot/.append style={ thick},
            ymax = 3000,
            legend style={at={(1,1)},{column sep=0.25cm}},
            legend entries={\ouralgo\ ($\epsilon=1$)}
        ]

\addplot[
    color=purple,
dotted, every mark/.append style={solid, fill=gray}, mark=triangle*,mark repeat = 1]
coordinates {
(10,2733)
(20,2579)
(30,2139)
(40,2106)
};
\end{axis}
\end{tikzpicture}
\caption{$c{_{ij}},q_{i} \sim \mathcal{N}(\alpha=0.4,0.2)$}
    \end{subfigure}
    \caption{\texttt{EXP3}: Average regret per agent with \ouralgo\ by varying the number of learners $n$ (with $\epsilon=1$, $t=100000$)}
    \label{fig:learners}
\end{figure}
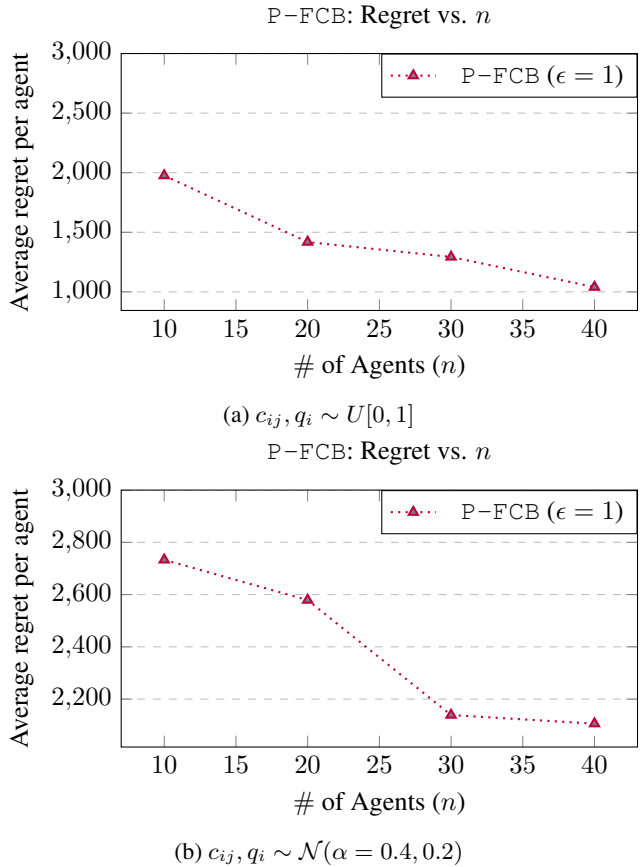


\subsection{Setup}

For our setting, we generate costs and qualities for the producers from: (a) uniform distributions, i.e., $q_{i},c_{ij} \sim U[0,1]$ (b) normal distributions, i.e., $q_{i},c_{ij} \sim \mathcal{N}(\alpha,0)$. For both cases, the capacities are sampled from a uniform distribution, $k_{ij} \sim U[1,50]$. We use the following tuning parameters in our experiments: $\alpha = 0.4$, $\delta = 0.01$ (i.e., $\delta< 1/n$), $\underline{t} = 200$, $\Bar{t} = 40000$, $\omega_{1} = 0.1$, $\omega_{2} = 10$. For our Oracle, we deploy the \textit{Greedy SSA} algorithm presented in Deva et al.~\cite{ayush}. Further, to compare \ouralgo's performance, we construct the following two \emph{non-private} baselines:

\begin{enumerate}
    \item \texttt{Non-Federated}. We use the single agent algorithm for subset selection under constraints proposed in Deva et al.~\cite{ayush}. It follows $UCB$ exploration similar to \ouralgo\ but omits any communication done with other agents.
    \item \texttt{FCB}. This is the non-private variant of \ouralgo. That is, instead of communicating $\Tilde{w}_{ij}$ and $\Tilde{y}_{ij}$, the true values $w_{ij}^{t}$ and $y_{ij}^{t}$ are communicated.  
\end{enumerate}

\noindent We perform the following experiments to measure \ouralgo's performance:

 \begin{itemize}
    \item[$\bullet$] \texttt{EXP1:} For fixed $n=10$, $m=30$, we observe the regret growth over rounds ($t$) and compare it to non-federated and non-private federated settings.
    \item[$\bullet$] \texttt{EXP2:} For fixed $n=10$, $m=30$, we observe $FRR$ (ratio of regret incurred in federated setting compared to non federated setting) at $t=100000$ while varying $\epsilon$ to see the regret variance w.r.t. privacy budget.
    \item[$\bullet$] \texttt{EXP3:} For fixed $\epsilon = 1$, $m=30$, we observe average regret at $t=100000$ for varying $n$ to study the effect of number of communicating agents.
\end{itemize}

For \texttt{EXP1} and \texttt{EXP2}, we generate $5$ instances by sampling costs and quality from both Uniform and Normal distributions. Each instance is simulated $20$ times and we report the corresponding average values across all instances. Likewise for \texttt{EXP3}, instances with same producer quality values are considered with costs and capacities defined for different numbers of learners. For each instance, we average across $20$ simulations.

\subsection{Results}
\begin{itemize}
    \item[$\bullet$] \texttt{EXP1}. \ouralgo\ shows significant improvement in terms of regret (Fig.~\ref{fig:exp1}) at the cost of relatively low privacy budget. Compared to \texttt{FCB}, \ouralgo\ ($\epsilon = 1$) and \texttt{Non-federated} incurs $136\%$,$233\%$ more regret respectively for uniform sampling and $235\%$, $394\%$ more regret respectively for normal sampling. This validates efficacy of \ouralgo.
    
  \item[$\bullet$] \texttt{EXP2}. We study the performance of the algorithm with respect to privacy budget (Fig.~\ref{fig:exp2}). We observe that according to our expectations, the regret decreases as privacy budget is increased. This decrease in regret is sub-linear in terms of increasing $\epsilon$ values. This is because as privacy budget increases, the amount of noise in communicated data decreases.
   
  \item[$\bullet$] \texttt{EXP3}. We see (Fig.~\ref{fig:learners}) an approximately linear decrease in per agent regret as the number of learning agents increases. This reinforces the notion of reduction of regret, suggested in Section~\ref{sec:np}, by engaging in federated learning is valid in a heterogeneous private setting.
   
\end{itemize}

\noindent\underline{Discussion}: Our experiments demonstrate that \ouralgo, through selective learning in a federated setting, is able to achieve a fair regret and privacy trade-off. \ouralgo\ achieves reduction in regret (compared to non-federated setting) for low privacy budgets. 


With regards to hyperparamters, note that lower $\omega_{2}$ suggests tighter bounds while selecting what to learn, implying a higher confidence in usefulness of the communicated data. Thus, larger values for $\omega_{1}$ can be used if $\omega_{2}$ is decreased. In general, our results indicate that it is optimal to maintain the value $\omega_{1}\cdot\omega_{2}$ used in our experiments. Also, the communication start time, should be such that the sampled noise is at-least a magnitude smaller than the accumulated uncommunicated data (e.g., $\underline{t}\approx 200$). This is done to ensure that the noisy data is not detrimental to the learning process. 

The DP-ML literature suggests a privacy budget $\epsilon<1$~\cite{triastcyn2019federated}. From Fig.~\ref{fig:exp2}, we note that \ouralgo\ performs well within this privacy budget. While our results achieve a fair regret and privacy trade-off, in future, one can further fine tune these hyperparameters through additional experimentation and/or theoretical analysis.



\section{Conclusion and Future Work}\label{sec:conclusion}
This paper focuses on learning agents which interact with the same set of producers (``arms") and engage in federated learning while maintaining privacy regarding their procurement strategies. We first looked at a non-private setting where different producers' costs and capacities were the same across all agents and provided theoretical guarantees over optimisation due to federated learning. We then show that extending this to a heterogeneous private setting is non-trivial, and there could be potential information leaks. We propose \ouralgo\, which uses \emph{UCB} based exploration while communicating estimates perturbed using Gaussian method to ensure differential privacy. We defined a communication protocol and a selection learning process using error bounds. This provided a meaningful balance between regret and privacy budget. We empirically showed notable improvement in regret compared to individual learning, even for considerably small privacy budgets. 


Looking at problems where agents do not share exact sets of producers but rather have overlapping subsets of available producers would be an interesting direction to explore. It is also possible to extend our work by providing theoretical upper bounds for regret in a differentially private setting. In general, we believe that the idea of when to learn and when not to learn from others in federated settings should lead to many interesting works.

\bibliographystyle{named}
\bibliography{ijcai22}

\end{document}